\documentclass{article}

 \pdfoutput=1
 
\usepackage{graphicx} 
\usepackage{subfigure} 

\usepackage{natbib}

\usepackage{algorithm}
\usepackage{algorithmic}

\usepackage{hyperref}

\usepackage[accepted]{icml2013}

\usepackage{amsmath}
\usepackage{amssymb}
\newcommand{\inner}[1]{\left\langle#1\right\rangle}

\def\R{\mathbb{R}}

\def\N{\mathbb{N}}

\newcommand{\norm}[1]{\left\|#1\right\|}
\newcommand{\abs}[1]{\left|#1\right|}

\def\ones{\mathbf{1}}

\def\assoc{\mathrm{assoc}}
\def\cut{\mathrm{cut}}

\def\vol{\mathop{\rm vol}\nolimits}

\def\argmin{\mathop{\rm arg\,min}\limits}

\def\subj{\mathop{\rm subject\,to}}

\def\ones{\mathbf{1}}

\def\cut{\mathrm{cut}}
\def\Ncut{\mathrm{NCut}}

\def\NCC{\mathrm{NCC}}

\def\Ct{C}

\newtheorem{theorem}{Theorem}
\newtheorem{lemma}{Lemma}

\newtheorem{definition}{Definition}

\newtheorem{proposition}{Proposition}

\newenvironment{proof}{\par\noindent{\bf Proof:\ }}{\hfill$\Box$\\[2mm]}
\newenvironment{proofa}[1]{\par\noindent{\bf Proof #1:\ }}{\hfill$\Box$\\[2mm]}

\newtheorem{thma}{Theorem}

\icmltitlerunning{Constrained fractional set programs}

\begin{document} 

\twocolumn[
\icmltitle{Constrained fractional set programs and their  \\ 
application in local clustering and community detection}

\icmlauthor{Thomas B\"uhler}{tb@cs.uni-saarland.de}
\icmladdress{Saarland University, Saarbr\"ucken, Germany}
\icmlauthor{Syama Sundar Rangapuram}{srangapu@mpi-inf.mpg.de}
\icmladdress{Max Planck Institute for Informatics \& Saarland University, Saarbr\"ucken, Germany}
\icmlauthor{Simon Setzer}{setzer@mia.uni-saarland.de}
\icmlauthor{Matthias Hein}{hein@cs.uni-saarland.de}
\icmladdress{Saarland University, Saarbr\"ucken, Germany}

\icmlkeywords{local clustering, community detection, network analysis, constrained fractional set programming, tight relaxation}

\vskip 0.3in
]

\begin{abstract} 
The (constrained) minimization of a ratio of set functions is a problem frequently occurring
in clustering and community detection. As these optimization problems are typically NP-hard, one
uses convex or spectral relaxations in practice. While these relaxations can be solved globally optimally,
they are often too loose and thus lead to results far away from the optimum. In this paper we show that \emph{every}
constrained minimization problem of a ratio of non-negative set functions allows a \emph{tight} relaxation into an unconstrained continuous
optimization problem. 
This result leads to a flexible framework for solving constrained problems in network analysis. 
While a globally optimal solution for the resulting non-convex problem cannot be guaranteed,
we outperform the loose convex or spectral relaxations by a large margin 
on constrained
local clustering problems.
\end{abstract}


\section{Introduction}
Graph-based data appear in manifold ways in learning problems - either the data have already graph structure as
in the case of 
social networks and biological networks 
or a similarity graph is constructed using a similarity measure based
on features of the data. 
Several graph-based problems in clustering and community detection can be modelled as the optimization
 of a ratio of set functions (referred to here as fractional set program).
Prominent examples are the normalized cut problem, from which the popular spectral clustering
method is derived \cite{SM00}, and the maximum density subgraph problem, which has applications in community detection
\cite{For2010} and
bioinformatics \cite{Sah2010}.

It turns out that in practice often 
 additional background or domain knowledge about the learning problem is available. Such
prior knowledge can then be incorporated as constraints into the optimization problem.
In the case of clustering, 
\citet{Wag01} are the first to show how prior information given in the form
of must-link and cannot-link constraints between vertices
can be integrated into the $k$-means algorithm.
Recently, \citet{RanHei12} proposed a generalization of
the normalized cut problem that can handle 
must-link and cannot-link constraints. 
In the recent work of \citet{MahOreVis12}, locality constraints in the form of a seed set and volume constraint have been integrated into the normalized
cut formulation.
Furthermore,  \citet{KhuSah09} and \citet{Sah2010} considered size and distance constraints for the maximum density subgraph problem.

Since the above-mentioned combinatorial problems are NP-hard, the standard approach is to consider convex or
spectral relaxations which can be solved globally optimally in polynomial time.
Due to its practical efficiency the spectral relaxation is very popular in machine learning, e.g.\;spectral clustering \cite{HK91, SM00}.
However, 
it
is often quite loose and thus leads to a solution far away from the optimal one of the original problem.
Moreover, spectral-type 
relaxations \cite{MahOreVis12} fail to guarantee that the constraints which encode the 
prior knowledge
are satisfied.

In another line of work \cite{HeiBue10,SB10,HeiSet11,BreLauUmiBre2012}, it has been shown that \textit{tight continuous relaxations} exist for
all balanced graph cut problems and the normalized cut subject to must-link and cannot-link constraints \cite{RanHei12}. A tight relaxation means that the continuous and the combinatorial optimization problem are equivalent in the sense that 
the optimal values agree and the optimal solution of the combinatorial problem can be obtained from the continuous solution.
While the resulting algorithms provide no guarantee to yield the globally optimal solution,
the standard loose relaxations are outperformed by a large margin in practice.

In this paper we show
that \textit{any} constrained minimization problem of a ratio of non-negative set functions allows a tight
relaxation into a continuous optimization problem. This result together with our efficient minimization techniques enables the
easy integration of prior information in form of constraints into many problems in graph-based clustering and community detection.
While the general framework introduced in this paper is applicable to all problems discussed so far, we will focus on
two particular applications: local clustering by constrained balanced graph cuts, and community detection via constrained densest subgraph problems.
Compared to previous work, the algorithms developed in this paper are the first to guarantee that all given constraints are fulfilled by the 
obtained solution. Note that in principle our method could also be applied to a setting with soft or noisy constraints, however we will focus here on the case of hard constraints. In the experimental section we will show the superior performance compared to state of the art methods \cite{AndLan06,MahOreVis12}.


\section{Fractional set programs in clustering and community detection}\label{sec:setratio_examples}

In the following, $G=(V,W)$ denotes an undirected, weighted graph with a non-negative, symmetric weight matrix $W \in \R^{n \times n}$, where $n=\abs{V}$. 
Moreover, by assigning a non-negative weight $g_i$ to each vertex $i$, we can define the general volume of a subset $A\subset V$ as $\vol_g(A) = \sum_{i\in A}g_i$. As special cases, we obtain for $g_i=1$ the cardinality $|A|$ and for $g_i$ equal to the degree $d_i = \sum_{j\in V}w_{ij}$ the classical volume $\vol(A)=\vol_d(A)$. Furthermore, $\overline{A}=V \backslash A$ denotes the complement of $A$. 

The balanced graph cut problem is a well-known problem in computer science with applications ranging from parallel computing to image segmentation \cite{PotSimLio1990,SM00}. A very popular balanced graph cut criterion is the normalized cut\footnote{This is up to a constant factor the same as the usual definition, $\Ncut(C,\overline{C})=\cut(C,\overline{C})\big(\frac{1}{\vol_d(C)}+\frac{1}{\vol_d(\overline{C})}\big)$.},
\[
\Ncut(C,\overline{C}) = \frac{\cut(C,\overline{C})}{\vol_d(C)\vol_d(\overline{C})} ,  \; \textrm{ for } C \subset V,
\]
where 
$\cut(C,\overline{C}):=\sum_{i\in C,j\in \overline{C}} w_{ij}$.
The spectral relaxation of the normalized cut leads to the popular spectral clustering method \cite{Lux07}. A related criterion 
is the normalized Cheeger cut,
\[
\NCC(C,\overline{C}) = \frac{\cut(C,\overline{C})}{\min\{\vol_d(C),\vol_d(\overline{C})\}}  ,  \; \textrm{ for } C \subset V.
\] 
More general balanced graph cuts were studied by \citet{HeiSet11}. 
In practice, often additional information about the desired solution is available which can be incorporated into the problem via constraints. 
This motivates us to consider a more general class of problems where one optimizes a ratio of set functions\footnote{A set function $\widehat{S}$ on a set $V$
is a function $\widehat{S}:2^V \rightarrow \R$.} subject to constraints.  
In the following, we discuss two examples of constrained problems in network analysis.

\paragraph{Constrained balanced graph cuts for local clustering.}
Recently, there has been a strong interest in balanced graph cut methods for local clustering. 
Starting with the work of \citet{SpiTen04}, initially, the goal was to develop an \textit{algorithm} that finds a subset near a given seed vertex with \textit{small} normalized cut or normalized Cheeger cut value with running time linear in the size of the obtained cluster. The proposed algorithm and subsequent work \cite{AndChuLan06, Chu09} 
 use random walks to explore the graph locally, without considering the whole graph.
Algorithms of this type have been applied for community detection in networks \cite{AndLan06}. 

In contrast, \citet{MahOreVis12} give up the runtime requirement and formulate the task as an explicit optimization problem, where one aims at finding the \textit{optimal} normalized cut subject to a seed constraint and an upper bound on the volume of the set containing the seed set. Again, the idea is to find
a local cluster around a given seed set.
Motivated by the standard spectral relaxation of the normalized cut problem, they derive a spectral-type relaxation which is biased towards solutions fulfilling the seed constraint. Their method has been successfully applied in semi-supervised image segmentation \cite{MajVisMal11} and for community detection around a given query set \citep{MahOreVis12}. However, while they provide an approximation guarantee for their relaxation, 
they cannot  guarantee that the returned solution satisfies seed and volume constraints.

In this paper we consider an extended version of the problem of \citet{MahOreVis12}. Let $J$ denote the set of seed vertices, $\widehat{S}$ a symmetric balancing function (e.g.\;$\widehat{S}(C)= \vol_d(C)\vol_d(\overline{C})$ for the normalized cut) and let $\vol_g(C)$ be the general volume of set $C$, where $g \in \R^n_+$ are vertex weights. 
The general local clustering problem can then be formulated as
\begin{align}\label{eq:local_clustering_general}
		\min_{C \subset V}&\;\frac{\cut(C,\overline{C})}{\widehat{S}(C)} \\
	 \subj: &\;\vol_g(C) \leq k, \; \textrm{ and }\; \notag
	 J \subset C.
\end{align}
The choice of the balancing
function $\widehat{S}$ allows the user to influence the trade-off between  getting 
a partition with small cut and a balanced
partition. 
One could also combine this with must- and cannot-link constraints (see \citealp{RanHei12}) or add even more complex constraints
such as an upper bound on the diameter of $C$. However, in order to compare to the
method of \citet{MahOreVis12},
we restrict ourselves in this paper to the normalized cut
with volume constraints,
that is $\widehat{S}(C)=\vol_d(C)\vol_d(\overline{C})$ and $g=d$.

\paragraph{Constrained local community detection.} 
A second related problem is constrained local community detection. In community detection it makes more sense to find a highly connected set instead of emphasizing the separation to the remaining part
of the graph by minimizing the cut. Thus, we are searching for a set $C$ which has high association, defined as $\assoc(C) = \sum_{i,j \in C} w_{ij}$. 
Dividing the association of $C$ by its size yields the density of $C$. The subgraph of maximum density can be computed in polynomial time \cite{Gol84}.
However, the obtained communities in the unconstrained problem are typically either too large or too small, which calls for size constraints. Note that the introduction of such constraints makes the problem NP-hard \cite{KhuSah09}.

A general class of (local) community detection problems can thus be formulated as 
\begin{align} \label{eq:maxdens_general}
   \max_{C \subset V} &\ \frac{\assoc(C)}{\vol_g(C)}\\ 
  \subj:  &\,k_1 \leq \vol_h(C) \le k_2,  \; \textrm{ and }\;J \subset C,	\nonumber
  \end{align}
  where $g,h \in \R^n_+$ are vertex weights. This formulation generalizes the above-mentioned density-based approaches by replacing the denominator by a general volume function $\vol_g$.
 One can use the vertex weights $g$ to
bias the obtained community towards one with desired properties by assigning small weights to vertices which one would prefer to occur in the solution
and larger weights to ones which are less preferred.

The problem \eqref{eq:maxdens_general}
with only lower bound constraints has been considered in
team selection \cite{GajSar2012} and bioinformatics \cite{Sah2010} where constant factor approximation algorithms were developed.
However, in the case of equality and upper bound constraints
the problem is very hard even 
when using only 
cardinality constraints (i.e., $h_i=1$),
 and it has been shown that there is no polynomial time approximation scheme in these cases \cite{Kho06, KhuSah09}.
Our method can handle such hard upper bound and equality constraints. In the experiments we show results 
for a community detection
problem with a specified query set $J$ and an upper bound on the size for a co-author network. 

Note that if $\vol_g(C)=\vol_d(C)$, one can decompose the objective of \eqref{eq:maxdens_general} analogously to the argument for the normalized cut \cite{SM00} as
  \[
  			\frac{\assoc(C)}{\vol_d(C)} 
  			= 1 -   \frac{\cut(C,\overline{C})}{\vol_d(C)}
  			\ .
  \] 
This implies that for $\vol_g(C)=\vol_d(C)$ in \eqref{eq:maxdens_general} and $\widehat{S}(C)=\vol_d(C)$ in \eqref{eq:local_clustering_general}, the problem \eqref{eq:maxdens_general} is equivalent to \eqref{eq:local_clustering_general} if we choose the same constraints. If one has only the constraint
$\vol_d(C) \leq \frac{1}{2}\vol_d(V)$ both problems are equivalent to the normalized Cheeger cut. 

\paragraph{Contributions of this paper.} 
We show that \textbf{all} constrained non-negative fractional set programs have an equivalent tight continuous relaxation. This general result enables the integration of prior information in form of constraints
into clustering and community detection problems. In particular, it allows us to derive efficient algorithms for problems \eqref{eq:local_clustering_general} and \eqref{eq:maxdens_general}.
  Our algorithms consistently outperform competing methods \cite{AndLan06,MahOreVis12}.
Moreover, we are not aware of any other methods for the above problems which can guarantee that the solution always satisfies volume \textbf{and} seed constraints.

Although the tight relaxation results in \citet{HeiSet11} and \citet{RanHei12}
encompass a large class of problems, they are not applicable to the problems considered in this paper
because of the following limitations:
First, tight relaxations were shown by \citet{HeiSet11} only for a ratio of \textit{symmetric} non-negative set functions, where the numerator is restricted to be \textit{submodular}.
We extend the results to \emph{arbitrary} ratios of non-negative set functions without any restrictions concerning symmetry or submodularity.
Second, only \emph{equality} constraints for \emph{non-negative} set functions restricted to be either \emph{submodular or supermodular} could be handled by \citet{RanHei12}.
       We generalize this to \emph{inequality} constraints\footnote{Note that $\widehat{M}(C)=k$ is equivalent to $k\leq\widehat{M}(C)\leq k$.}
      without any restrictions on the constraint set functions in order to handle the constraints in \eqref{eq:local_clustering_general} and \eqref{eq:maxdens_general}.


\section{Tight relaxations of fractional set programs with constraints}\label{sec:relax}

The problems discussed in the last section can be written in the following general form:
\begin{align}\label{eq:set_ratio}
   \min_{C \subset V} &\;\frac{\widehat{R}(C)}{\widehat{S}(C)} =: \widehat{Q}(C)\\
  \subj:  &\;\widehat{M}_i(C) \le k_i, \quad i=1,\ldots,K \nonumber
  \end{align}
where $\widehat{R},\widehat{S},\widehat{M}_i:2^V \rightarrow \R$ are set functions on a set $V= \left\{ 1,\dots,n\right\}$. We assume here that $\widehat{R},\widehat{S}$ are non-negative and that $\widehat{R}(\emptyset)=\widehat{S}(\emptyset)=0$. 
No assumptions are made on the set functions $\widehat{M}_i$, in particular they are not required to be non-negative. Thus also lower bound constraints can be written in the above form.
Moreover, the formulation in \eqref{eq:set_ratio} also encompasses the subset constraint $J\subset C$ in \eqref{eq:local_clustering_general} and \eqref{eq:maxdens_general} as it can be written as  equality constraint $\abs{J} - \abs{J \cap C} = 0$. Alternatively, we will discuss a direct integration of the subset constraint into the objective in Section \ref{sec:max-density}.

The connection between the set-valued and the continuous space is achieved via thresholding. 
Let $f\in \R^n$, and we assume wlog that $f$ is ordered 
in ascending order 
$f_1\leq f_2\leq\cdots\leq f_n$. One defines the sets
\begin{equation}\label{eq:sets}
	\Ct_i:= \left\{ j \in V | f_j \geq f_i\right\} , \hspace{1cm}  i=1, \dots ,n.
\end{equation}
	We frequently make use of this notation in the following. Furthermore, we use $\ones_C \in \R^n$ to denote the indicator vector of the set $C$, i.e.\;the vector which is $1$ at entry $j$ if $j \in C$ and $0$ otherwise.
A key tool for the derivation of the results of this paper is the Lovasz extension as a way to extend a set function (seen as function on the hypercube) to a function on $\R^n$. 
\begin{definition} \label{def:lovasz}
Let $\widehat{R}:2^V \rightarrow \R$ be a set function with $\widehat{R}(\emptyset)=0$, and $f \in \R^n$ 
in ascending order 
$f_1\leq f_2\leq\cdots\leq f_n$. 
The Lovasz extension $R:\R^n \rightarrow \R$ of $\widehat{R}$ 
is
 defined as
	$R(f) 
	 = \sum_{i=1}^{n-1} \widehat{R}(\Ct_{i+1}) \left(f_{i+1}-f_i\right) + \widehat{R}(V) f_1.$
\end{definition}
Note that $R(\ones_C)=\widehat{R}(C)$ for all $C \subset V$, i.e.\;$R$ is indeed an extension of $\widehat{R}$ from $2^V$ to $\R^n$. In the following, we always use the hat-symbol $(\;\widehat{}\;)$ to denote set functions and omit it for the corresponding Lovasz extension.
A particular important class of set functions are submodular set functions since their Lovasz extension is convex \cite{Bach11}.
\begin{definition}
A set function $\widehat{R}:2^V \rightarrow \R$ is submodular if for all $A,B \subset V$,
$	\widehat{R}(A\cup B) + \widehat{R}(A \cap B)  \leq \widehat{R}(A) + \widehat{R}(B)$. 
It is supermodular, if the converse inequality holds true, 
and modular if we have equality.
\end{definition}

The connection between submodular set functions and convex functions is as follows (see \citealp{Bach11}).
\begin{proposition}\label{prop:convexity_submodularity}
Let $R:\R^V \rightarrow \R$ be the Lovasz extension of $\widehat{R}:2^V \rightarrow \R$. Then, $\widehat{R}$ is submodular if and only if $R$ is convex.
Furthermore, if $\widehat{R}$ is submodular, then $\min_{A\subset V} \widehat{R}(A) = \min_{f\in \left[0,1 \right]^n} R(f)$.
\end{proposition}
Thus submodular minimization problems reduce to convex minimization problems. A similar equivalence of continuous and combinatorial
optimization problems is the main topic of this paper. In the following we list some useful properties of the Lovasz extension (see \citealp{Fuj2005,Bach11,HeiSet11}).

\begin{proposition}\label{prop:lovasz_prop}
Let $R:\R^V \rightarrow \R$ be the Lovasz extension of $\widehat{R}:2^V \rightarrow \R$. Then,
\begin{itemize}
	\item $R$ is positively one-homogeneous\footnote{$R:\R^V \rightarrow \R$ is positively one-homogeneous if $R(\alpha f)=\alpha \,R(f),\; \forall \alpha \in \R$ with $\alpha\geq 0$.},
	\item $R(f)\geq 0, \; \forall\, f \in  \R^V$ 
	 and $R(\ones)=0$ if and only if $\widehat{R}(A)\geq 0, \; \forall \, A \subset V$ and $\widehat{R}(V)=0$,
	\item Let $S:\R^V \rightarrow \R$ be the Lovasz extension of $\widehat{S}:2^V \rightarrow \R$. Then, $\lambda_1\,R+\lambda_2\,S$ is the Lovasz extension of $\lambda_1\,\widehat{R} + \lambda_2\,\widehat{S}$, for all $\lambda_1,\lambda_2 \in \R$.
\end{itemize}
\end{proposition}

\paragraph{Unconstrained fractional set programs.}
Using the property of the Lovasz extension that  $R(\ones_C)=\widehat{R}(C)$ for all $C \subset V$, one can directly observe
that the following continuous fractional program is a relaxation of the unconstrained version of problem \eqref{eq:set_ratio}
\[ 
\inf_{f \in \R_+^n} \frac{R(f)}{S(f)}.
\]
The following theorem shows that the relaxation is in fact tight, in the sense that the optimal values agree and the solution of the set-valued problem can be computed from the solution of the continuous problem.

Note that given a vector $f\in \R^n$ for the continuous problem, one can construct a set $C'$ by computing
\[
	C' =  \argmin_{C_i,i=1,\ldots,n} \frac{\widehat{R}(\Ct_{i})} {\widehat{S}(\Ct_{i})},
\]
where the sets $\Ct_{i}$ are defined in \eqref{eq:sets}. We refer to this process as \textit{optimal thresholding}.

\begin{theorem}\label{thm:ratio_setfunct}
Let $\widehat{R},\widehat{S}:2^V \rightarrow \R$ be non-negative set functions and  $R,S:\R^n \rightarrow \R$ their Lovasz extensions, respectively. Then, it holds that
\[
	\inf_{C \subset V} \frac{\widehat{R}(C)}{\widehat{S}(C)} = \inf_{f \in \R_+^n} \frac{R(f)}{S(f)} \ .
\]
Moreover, it holds for all $f\in \R^n_+$, 
	$\frac{R(f)}{S(f)} \geq  \min_{i=1,\ldots,n} \frac{\widehat{R}(\Ct_{i})} {\widehat{S}(\Ct_{i})}$.
Thus a minimizer of the set ratio can be found by optimal thresholding.
Let furthermore $\widehat{R}(V)=\widehat{S}(V)=0$, then all the above statements hold if one replaces $\R^n_+$ with $\R^n$.
\end{theorem}
 In practice it may sometimes by difficult to derive and/or work with explicit forms of the Lovasz extensions of $\widehat{R}$ and $\widehat{S}$.
 However, the following more general version of Theorem \ref{thm:ratio_setfunct} shows that, given a decomposition of $\widehat{R}$ and $\widehat{S}$ into a difference of submodular set functions, 
 one needs
the Lovasz extension only for the
first 
term of $\widehat{R}$ and 
the second 
term of $\widehat{S}$.
The remaining terms can be replaced by any convex one-homogeneous functions that also extend the corresponding set functions.
Note that by Proposition \ref{thm:setfct_dc} such a decomposition always exists.

\begin{thma}[b]\label{thm:ratio_setfunct_supp}
Let $\widehat{R},\widehat{S}:2^V \rightarrow \R$ be non-negative set functions and $\widehat{R}:=\widehat{R}_1-\widehat{R}_2$ and $\widehat{S}:=\widehat{S}_1-\widehat{S}_2$  be decompositions into differences of submodular set functions. 
Let the Lovasz extensions of $\widehat{R}_1,\widehat{S}_2$ be given by $R_1,S_2$ and let $R'_2, S'_1$ be positively one-homogeneous convex functions with $S'_1(\ones_A)=\widehat{S}_1(A)$ and $R'_2(\ones_A)=\widehat{R}_2(A)$ such that $S'_1 -S_2$ is non-negative. Define $R:=R_1-R'_2$ and 
$S:=S'_1-S_2$. Then, 
\[
	\inf_{C \subset V} \frac{\widehat{R}(C)}{\widehat{S}(C)} = \inf_{f \in \R_+^V} \frac{R(f)}{S(f)} \ .
\]
Moreover, it holds for all $f\in \R^n_+$, 
	$
	\frac{R(f)}{S(f)} \geq  \min_{i=1,\ldots,n} \frac{\widehat{R}(\Ct_{i})} {\widehat{S}(\Ct_{i})}
	$.
Thus a minimizer of the set ratio can be found by optimal thresholding.
Let furthermore $\widehat{R}(V)=\widehat{S}(V)=0$, then all the above statements hold if one replaces $\R^n_+$ with $\R^n$.
\end{thma}
Before we prove the above Theorem, we collect some useful results. Lemma \ref{lem:setfunct} shows that the Lovasz extension of a submodular set function $\widehat{R}$ is an upper bound on any one-homogeneous convex function $R'$ which extends the set function $\widehat{R}$ to the continuous space.
\begin{lemma}\label{lem:setfunct}
Let $\widehat{R}:2^V \rightarrow \R$ be a submodular set function with $\widehat{R}(\emptyset)=0$. 
Let $R'$ be a positively one-homogeneous convex function with $R'(\ones_A)=\widehat{R}(A)$ for all $A \subset V$. 
Then, it holds $\forall f\in \R_+^V$ that 
\[
	R'(f) \leq \sum_{i=1}^{n-1} \widehat{R}(\Ct_{i+1}) \left( f_{i+1}-f_i\right) + f_1 \widehat{R}(V).
\]
Let furthermore $\widehat{R}(V)=0$, then the above inequality holds for all $f\in \R^V$.
\end{lemma}
\begin{proof}
Let $f$ be ordered in increasing order $f_1 \leq f_2 \leq \dots \leq f_n$. Note that every convex, positively one-homogeneous function  $R':\R^V \rightarrow \R$ can be written as $R(f)= \sup_{u \in U} \inner{u,f}$,
where $U$ is a convex set (see \citealp{HirLem2001}). Then, since for any $u \in U$, $\inner{u,f} \leq R'(f)$, it holds that
\[
	\widehat{R}(\Ct_{i}) = R'(\ones_{\Ct_{i}}) \geq  \inner{u,\ones_{\Ct_{i}}}, \quad i=1,\ldots,n,
\]
for any $u \in U$ and hence for all  $f\in \R_+^V$,  
\begin{align} \label{eq:lovasz-lower-bound}
	& \sum_{i=1}^{n-1} \widehat{R}(\Ct_{i+1}) \left( f_{i+1}-f_i\right)   + f_1 \widehat{R}(V) \notag \\
	&\geq \sum_{i=1}^{n-1} \inner{u,\ones_{\Ct_{i+1}}} \left( f_{i+1}-f_i\right)  + f_1 \inner{u,\ones} \notag \\
    &= \sum_{i=1}^n f_i u_i. 
	\end{align}
As this holds for all $u \in U$ we obtain for all $f\in \R_+^V$,  
\begin{align*}
	\sum_{i=1}^{n-1} \widehat{R}(\Ct_{i+1}) \left( f_{i+1}-f_i\right) +  \widehat{R}(V)   f_1  \geq \sup_{u \in U} \inner{f,u} =R'(f) \ .
\end{align*}
For the second statement we use the fact that with the condition $\widehat{R}(V)=0$ the lower bound in \eqref{eq:lovasz-lower-bound}
holds for all $f \in \R^V$.
\end{proof}
The main part of the proof of Theorem \ref{thm:ratio_setfunct_supp} (b) is the following Lemma which
implies that optimal thresholding of a vector $f$ always leads to non-increasing values of $R(f)/S(f)$. 
\begin{lemma}\label{lemma:thresholding}
Let $\widehat{R},\widehat{S}:2^V \rightarrow \R$ and $R,S:\R^n \rightarrow \R$ satisfy the assumptions of Theorem \ref{thm:ratio_setfunct} (b). Then for all $f \in \R_+^V$, 
\[
	\frac{R(f)}{S(f)} \geq  \min_{i=1,\ldots,n} \frac{\widehat{R}(\Ct_{i})} {\widehat{S}(\Ct_{i})}  \ .
\]
Let furthermore $\widehat{R}(V)=\widehat{S}(V)=0$, then the result holds for all $f \in \R^V$.
\end{lemma}
\begin{proof}
Let $R_1,S_2$ and $R'_2, S'_1$ satisfy the conditions from Theorem \ref{thm:ratio_setfunct} (b). Let furthermore $R_2$ and $S_1$ be the Lovasz extensions of $\widehat{R}_2$ and $\widehat{S}_1$.
With Lemma \ref{lem:setfunct} 
and Def.\,\ref{def:lovasz}, we get $\forall f \in \R^n_+$, 
 \begin{align*}
  & R(f)  = R_1(f) - R'_2(f) \geq R_1(f)-R_2(f) \\
  &=  \sum_{i=1}^{n-1} \widehat{R}(\Ct_{i+1}) \left( f_{i+1}-f_i\right) +f_1 \widehat{R}(V)\\
  &= \hspace{-0.05cm} \sum_{i=1}^{n-1} \hspace{-0.05cm}\frac{\widehat{R}(\Ct_{i+1})} {\widehat{S}(\Ct_{i+1})} \widehat{S}(\Ct_{i+1}) \left( f_{i+1}\hspace{-0.02cm}-\hspace{-0.02cm}f_i\right) \hspace{-0.02cm}+\hspace{-0.02cm}  \frac{\widehat{R}(V)} {\widehat{S}(V)} \widehat{S}(V) f_1 \\
	   &\geq \hspace{-0.05cm}\min_{j=1,\ldots,n} \hspace{-0.05cm}\frac{\widehat{R}(\Ct_{j})} {\widehat{S}(\Ct_{j})} \hspace{-0.05cm} \left(\sum_{i=1}^{n-1} \widehat{S}(\Ct_{i+1}) \left( f_{i+1}\hspace{-0.05cm}-\hspace{-0.05cm}f_i\right) \hspace{-0.05cm}+ \hspace{-0.05cm}f_1 \widehat{S}(V)\hspace{-0.05cm}\right) 
	   \end{align*}
	    where we used the non-negativity of $\widehat{R}$ and $\widehat{S}$ as well as the fact that $f\in \R^n_+$. Again using Def.\,\ref{def:lovasz}, the above is equal to
	   \begin{align*} 
	   & \hspace{-0.05cm} \min_{j=1,\ldots,n}\hspace{-0.05cm} \frac{\widehat{R}(\Ct_{j})} {\widehat{S}(\Ct_{i})} \left( S_1(f) - S_2(f)\right)\\
	   & \geq \hspace{-0.05cm} \min_{j=1,\ldots,n}\hspace{-0.05cm} \frac{\widehat{R}(\Ct_{j})} {\widehat{S}(\Ct_{i})} \left(S'_1(f)-S_2(f)\right). 
\end{align*}
By assumption, $S'_1-S_2$ is non-negative and thus division gives the result. The second statement is shown analogously.
\end{proof}
Now we are ready to prove Theorem \ref{thm:ratio_setfunct_supp} (b).
\begin{proofa}{of Theorem \ref{thm:ratio_setfunct_supp} (b)}
Lemma \ref{lemma:thresholding} implies that
	\[
	\inf_{f \in \R_+^V} \frac{R(f)}{S(f)}
	\geq \inf_{f \in \R_+^V} \;\min_{\substack{C_i \textrm{ def. by } f\\i=1,\ldots,n}} \,\frac{\widehat{R}(\Ct_{i})}  {\widehat{S}(\Ct_{i})} \geq \inf_{A \subset V} \frac{\widehat{R}(A)}{\widehat{S}(A)} \ .
\]
On the other hand we have
\[
	\inf_{A \subset V} \frac{\widehat{R}(A)}{\widehat{S}(A)} = \inf_{A \subset V} \frac{R(\ones_A)}{S(\ones_A)} \geq \inf_{f \in \R_+^V} \frac{R(f)}{S(f)}  \ , 
	\]
which implies equality. The statement regarding optimal thresholding has been shown in Lemma \ref{lemma:thresholding}. The proof for the case where $\widehat{R}(V)=\widehat{S}(V)=0$ works analogously.
\end{proofa}
Note that no assumptions except non-negativity are made on  
$\widehat{R}$ and $\widehat{S}$ - \emph{every} non-negative fractional set program has a tight relaxation into a continuous fractional program. 
The efficient minimization of the continuous objective will be the topic of Section \ref{sec:alg}.

\paragraph{Constrained fractional set programs.}
To solve the constrained fractional set program \eqref{eq:set_ratio} we make use of the concept of \textit{exact penalization} \cite{Dip94}, where the main idea is to transform a given constrained optimization problem into an {\it equivalent} unconstrained one by adding a penalty term.
We use the same idea for our constrained fractional set programs and define the penalty set function for a constraint $\widehat{M}_i(C)\leq k_i$ as
 \begin{align}\label{eq:penality_functions}
 \widehat{T}_i(C) & =\left\{ \begin{array}{cc}
 									\max\left\{ 0,\widehat{M}_i(C) - k_i\right\}, & C \neq \emptyset,\\
 									0, & C=\emptyset.\\
 									\end{array} \right. 	 
\end{align}
The function $\widehat{T}_i(C)$ is zero if $C$ is feasible for the $i$-th constraint and otherwise increasing with increasing infeasibility. The special treatment of the empty set in the definition of $\widehat{T}_i$ is a technicality required for the Lovasz extension. Defining 
$\widehat{T}(C) := \sum_{i=1}^K \widehat{T}_i(C)$, we can now formulate a modified problem
\begin{align} \label{mod_problem}
    \min_{C \subset V} \frac{ \widehat{R}(C) + \gamma \sum_{i}^K\widehat{T}_i(C)} {\widehat{S}(C)}\, =: \widehat{Q}_\gamma(C).
  \end{align}
We will show that using a feasible set of \eqref{eq:set_ratio}
 one can compute a $\gamma$ such that \eqref{mod_problem} is equivalent to the original constrained problem. Once we have established the equivalence, we can then apply Theorem \ref{thm:ratio_setfunct}, noting that $\widehat{T}$ is a non-negative set function.
This leads to the main result of this paper showing a tight relaxation of \emph{all} problems of form \eqref{eq:set_ratio}
where $\widehat{R},\widehat{S}$ are non-negative set functions.
In the following, the constant $\theta$ quantifies a ``minimum value'' of $\widehat{T}_i$ on the infeasible sets:
\[ \theta = \min_{i=1,\ldots,K} \big[ \min_{\widehat{M}_i(C)>k_i} \widehat{M}_i(C) - k_i\big].\]
For example, if $\widehat{M}(C)=\abs{C}$, then $\theta$ is equal to $1$. If $\widehat{M}(C)=\vol_g(C)$ and all vertex weights $g_i$ are rational numbers which are multiples of a fraction $\frac{1}{\rho}, \rho\in \N$, then $\theta\geq \frac{1}{\rho}$. Note that in practice, the constant $\theta$ and the parameter $\gamma$ introduced in the following are never explicitly computed (see experimental section).
  \begin{theorem}\label{th:setratio_cnstr}
  Let $\widehat{R},\widehat{S}:2^V \rightarrow \R$ be non-negative set functions and
 $R$, $S$ their Lovasz extensions. 
  Let $C_{0}\subset V$ be feasible and $\widehat{S}(C_0)>0$. Denote by $T$ the Lovasz extension of $\widehat{T}$. 
  Then, for $\gamma >   \frac{\widehat{R}(C_0)}{\theta\widehat{S}(C_0)}\, \max_{C \subset V} \widehat{S}(C)$, 
  \begin{align*}
    \min_{\substack{\widehat{M}_i(C)\leq k_i,\\i=1,\ldots,K}} \frac{\widehat{R}(C)}{\widehat{S}(C)}
					= \min_{f \in \R_{+}^n} \frac{R(f) + \gamma\,T(f)}{S(f)}:=Q_\gamma(f)
  \end{align*}
Moreover, for any $f \in \R_+^n$ with $Q_\gamma(f)<\widehat{Q}_\gamma(C_0)$ for the given $\gamma$, we have $Q_\gamma(f) \geq  \min_{i=1,\ldots,n} \widehat{Q}_\gamma(\Ct_{i})$,
and the minimizing set on the right hand side is feasible.
  \end{theorem}
\begin{proof}
We will first show the equivalence between the constrained fractional set program \eqref{eq:set_ratio} and the unconstrained problem \eqref{mod_problem} 
for the given choice of $\gamma$. Then the equivalence to the continuous problem will follow by Theorem \ref{thm:ratio_setfunct}.

 Define $\widehat{T}(C) := \sum_{i=1}^K \widehat{T}_i(C)$. 
  	Note that for any feasible subset $C$, that is $\widehat{M}_i(C) \leq k_i$, $i=1,\ldots,K$, the objective $Q_\gamma$ of problem \eqref{mod_problem} is equal to the objective $Q$ of problem \eqref{eq:set_ratio}. Thus, if we show that \textit{all} minimizers of the second problem satisfy the constraints then the equivalence follows. Suppose that $C^{*}\neq \emptyset$ is a minimizer of the second problem and that 
 	$C^*$ is infeasible. Then by definition we have $\widehat{T}(C^*)\geq \theta$. This yields
		\begin{align}	\label{eq:derivation_proof_penalty}
		\widehat{Q}_\gamma(C^{*}) &= \frac{\widehat{R}(C^{*}) + \gamma \widehat{T}(C^*)} { \widehat{S}(C^{*}) } \\
				       &\ge \frac{ \gamma  \widehat{T}(C^*)} { \widehat{S}(C^{*}) } 
				      \hskip-0.07cm \ge \hskip-0.07cm \frac{ \gamma  \widehat{T}(C^*)} { \max_{C \subset V} \widehat{S}(C) } \hskip-0.07cm \ge \hskip-0.07cm \frac{ \gamma  \theta}{ \max_{C \subset V} \widehat{S}(C) }, \notag
		\end{align}
		where we used the non-negativity of $\widehat{R}$ and $\widehat{S}$. Hence
		\begin{align*}
		\widehat{Q}_\gamma(C^{*}) 
&\ge \frac{ \gamma \theta} { \max_{C \subset V} \widehat{S}(C)} > \frac{\widehat{R}(C_0)}{\widehat{S}(C_0)}=\widehat{Q}_\gamma(C_{0}),
		\end{align*}
		which contradicts the fact that $C^{*}$ is optimal.
		
  Noting that $\widehat{T}$ is a non-negative function with $\widehat{T}(\emptyset)=0$ and $\gamma>0$, we have a ratio of non-negative set functions which attain the value zero
  on the empty set. Thus application of Theorem \ref{thm:ratio_setfunct} yields the equivalence to the continuous problem.
  
  The second statement can be seen as follows. Suppose $Q_\gamma(f) < \widehat{Q}_\gamma(C_0)$. By Lemma \ref{lemma:thresholding} we obtain
  \[ Q_\gamma(f) \geq \min_{i=1,\ldots,n} \widehat{Q}_\gamma(\Ct_{i}).\]
  Now suppose that the minimizer $C^*$ of the right hand side is not feasible, then again by the derivation in \eqref{eq:derivation_proof_penalty} 
   and the choice of $\gamma$,
  \[ \widehat{Q}_\gamma(C^*) \ge \frac{ \gamma  \theta} { \max_{C \subset V} \widehat{S}(C)} > \widehat{Q}_\gamma(C_0),\]
  which leads to a contradiction. Thus $C^*$ is feasible.
  \end{proof}
Note that Theorem \ref{th:setratio_cnstr} implies that the set found by optimal thresholding of the solution of the continuous program is guaranteed to satisfy all constraints. We are not aware of any other method which can give the same guarantee for the problems \eqref{eq:local_clustering_general} and \eqref{eq:maxdens_general}.


\section{Minimization of the tight continuous relaxation} \label{sec:alg}
The continuous optimization problems in Theorems \ref{thm:ratio_setfunct} and \ref{th:setratio_cnstr} have the form 
\begin{align}\label{opt:ratio} 
  	\min_{f \in \R^{n}_{+}} \frac{R(f)}{S(f)}:=Q(f),	
\end{align}
where $R$ and $S$ are non-negative. The fact that they are the Lovasz extensions of set functions $\widehat{R}, \widehat{S}$ also implies that they are one-homogeneous, see \citet{Bach11}. We now apply a slightly modified version of a result from \citet{HeiSet11}.
\begin{proposition}\label{thm:setfct_dc}
Every set function $\widehat{S}$ with $\widehat{S}(\emptyset)=0$ can be written as $\widehat{S}=\widehat{S}_1-\widehat{S}_2$, where $S_1$ and $S_2$ are submodular and $\widehat{S}_1(\emptyset)=\widehat{S}_2(\emptyset)=0$. The Lovasz extension $S$ can be written as difference of convex functions.
\end{proposition}

The above result implies that \eqref{opt:ratio} can be written as ratio of differences of convex functions (d.c.), i.e.\;$R = R_{1} - R_{2}$ 
with $R_{1}$, $R_{2}$ convex, and similarly for $S$. 
As the proof of Proposition \ref{thm:setfct_dc} is constructive, the explicit form of this decomposition can be calculated. 
We can now use a modification of the RatioDCA which has recently been proposed as an algorithm for minimizing  a non-negative ratio of one-homogeneous d.c.\;functions \cite{HeiSet11}.
This modification is necessary as the
problems in Theorem \ref{thm:ratio_setfunct} and \ref{th:setratio_cnstr} require optimization over the positive orthant.
We report the modified version in order to make the paper self-contained.

\floatname{algorithm}{}
\begin{algorithm}[htb]
   \renewcommand{\thealgorithm}{}
   \caption{\textbf{RatioDCA} Minimization of a non-negative ratio of one-homogeneous d.c functions over $\R^n_+$}
   \label{alg:ratio_dc}
\begin{algorithmic}[1]
   \STATE {\bfseries Initialization:} $f^0 \in \R^n_+$, 
   $\lambda^0 = Q(f^0)$
   \REPEAT
   \STATE  $f^{l+1} =
   	 \argmin_{u \in \R^{n}_{+},\, \norm{u}_2 \leq 1} \left\{ R_1(u) - \inner{u, r_2(f^l)} \right.$ \\
   	  \hspace{2cm}
   	  $\left. 
   	  + \lambda^l \big( S_{2}(u) -\inner{u, s_{1}(f^l)} \big)   \right\}$ \\
   	 	\text{where\ } $r_2(f^l) \in \partial R_2(f^l)$, $s_{1}(f^l) \in \partial S_{1}(f ^l)$
   \STATE $\lambda^{l+1}= Q(f^{l+1})$
	\UNTIL $\frac{\abs{\lambda^{l+1}-\lambda^l}}{\lambda^l}< \epsilon$
\end{algorithmic}
\end{algorithm} 
We will refer to the convex optimization problem 
solved at each step (line 3) as the {\it inner problem}. 
\begin{proposition}\label{prop:decrease_algorithm}
The sequence $f^{l}$ produced by RatioDCA satisfies $Q(f^{l+1}) < Q(f^{l})$ for all $l \ge 0$ or the sequence terminates.
\end{proposition}
 \begin{proof}Let
 $ \Phi_{f^{l}}(u) := R_{1}(u) - \inner{u, r_{2}(f^{l})} + \lambda^{l} \big(S_{2}(u) - \inner{u, s_{1}(f^{l})}\big)$
 denote the objective of the inner problem.
  The optimal value of the inner problem is non-positive since 
  \begin{align*}
  \Phi_{f^{l}}(f^{l}) &= R_{1}(f^{l}) - \inner{f^{l}, r_{2}(f^{l})} \\
  & \hskip0.4cm + \lambda^{l} \big(S_{2}(f^{l}) - \inner{f^{l}, s_{1}(f^{l})}\big)\\
  			     &= R_{1}(f^{l}) - R_{2}(f^{l}) + \lambda^{l} \big(S_{2}(f^{l}) - S_{1}(f^{l})\big) = 0,
  \end{align*}
   where we used the fact that $\inner{f^{l}, r_{2}(f^{l})} = R_{2}(f^{l})$ and $\inner{f^{l}, s_{1}(f^{l})} = S_{1}(f^{l})$. 
   Since $\Phi_{f^{l}}$ is one-homogeneous, the minimum of $\Phi_{f^{l}}$ is always attained at the boundary of the constraint
    set. If the optimal value is zero, then $f^{l}$ is a possible minimizer and the sequence terminates. 
    Otherwise the optimal value is negative and at the optimal point we get 
    \begin{align*}
     0&   > \Phi_{f^{l}}(f^{l+1})   \\
     	&= R_{1}(f^{l+1}) - \inner{f^{l+1}, r_{2}(f^{l})} \\
     	& \hskip0.5cm + \lambda^{l} \big(S_{2}(f^{l+1}) - \inner{f^{l+1}, s_{1}(f^{l})}\big)\\
	&\ge R_{1}(f^{l+1}) - R_{2}(f^{l+1}) + \lambda^{l} \big(S_{2}(f^{l+1}) - S_{1}(f^{l+1})\big),
    \end{align*}
    where we used that for a positively one-homogeneous convex function one has for all $f, g \in \R^{n}_{+}$, 
    \[ S(f) \ge S(g) + \inner{f-g, s(g)} = \inner{f, s(g)}. \]  
    Thus we obtain 
    \[ Q(f^{l+1}) = \frac{R_{1}(f^{l+1}) - R_{2}(f^{l+1})} {S_{1}(f^{l+1}) - S_{2}(f^{l+1})}
    		      < \lambda^{l} = Q(f^{l}). \]
 \end{proof}
The norm constraint of the inner problem is necessary as otherwise the problem would be unbounded from below. However, the choice of the norm plays no role in the proof and any norm can be chosen.
  Moreover, in the special case where the one-homogeneous function $R$ is convex and $S$ is concave, 
the RatioDCA reduces to Dinkelbach's method from fractional programming \cite{Din1967} and therefore computes the global optimum. In the general case, convergence to the global optimum cannot be guaranteed.
However, 
we can provide a \textit{quality guarantee}: RatioDCA either improves a given feasible set or stops after one iteration.
\begin{theorem}\label{th:quality-guarantee}
Let $A$ be a feasible set and $\gamma > \widehat{R}(A) \max_{C \subset V} \widehat{S}(C)/(\theta\,\widehat{S}(A))$. 
Let $f^*$ denote the result of RatioDCA after initializing with the vector $\ones_A$, and let $C_{f^*}$ denote the set found by optimal thresholding of $f^*$. 
Either RatioDCA terminates after one iteration, or $C_{f^*}$ is feasible and $\frac{\widehat{R}(C_{f^*})}{\widehat{S}(C_{f^*})} \,< \,\frac{\widehat{R}(A)}{\widehat{S}(A)}$.
\end{theorem}
\begin{proof}
Proposition \ref{prop:decrease_algorithm} implies that the RatioDCA either directly terminates or produces a strictly monotonically decreasing sequence. In the latter case, using the strict monotonicity 
and the fact that thresholding does not increase the objective (Lemma \ref{lemma:thresholding}), we obtain
 \begin{align*}
 	\widehat{Q}_\gamma(A) = Q_\gamma(\ones_A)  
 	& \stackrel{Prop.\,\ref{prop:decrease_algorithm}}{>} Q_\gamma(f^*)\\
 	& \stackrel{Lemma\,\ref{lemma:thresholding}}{\geq}    Q_\gamma(\ones_{C_{f^*}}) = \widehat{Q}_\gamma(C_{f^*}) \ . 
 \end{align*}
Assume now that $C_{f^*}$ is infeasible. Then, one can derive analogously to the proof of Theorem \ref{th:setratio_cnstr} that $\widehat{Q}_\gamma(C_{f^*}) \geq \frac{\gamma \theta}{\max_{C\subset V} \widehat{S}(C)} > \widehat{Q}(A)=\widehat{Q}_\gamma(A)$, which is a contradiction to $\widehat{Q}_\gamma(A) > \widehat{Q}_\gamma(C_{f^*})$. Hence, $C_{f^*}$ has to be feasible and it holds that $\widehat{Q}(A)=\widehat{Q}_\gamma(A) > \widehat{Q}_\gamma(C_{f^*}) = \widehat{Q}(C_{f^*})$.
\end{proof} 
The above theorem implies that all constraints of the original constrained fractional set program are fulfilled by the set $C_{f^*}$ returned by RatioDCA.


\section{Tight relaxations of constrained maximum density and constrained balanced graph cut problems}\label{sec:max-density}

The framework introduced in this paper allows us to derive tight relaxations of all problems discussed in Section \ref{sec:setratio_examples}.
In the following, we will derive a tight relaxation of the local community detection problem 
\begin{align} \label{eq:maxdens_general_only_upper}
   \max_{C \subset V} &\ \frac{\assoc(C)}{\vol_g(C)}\\ 
  \subj:  &\,\vol_h(C) \le k,  \; \textrm{ and }\;J \subset C.	\nonumber
  \end{align}
For the constrained balanced graph cut problem, the tight relaxation can be found in a very similar way and is thus omitted here. 

First, we integrate the volume constraint via a penalty term, see \eqref{mod_problem}, which yields the equivalent problem
\begin{align} \label{maxdens_seed_constraint}
& \min_{\stackrel{C \subset V}{{\rm s.\hskip-0.02cm t.}\hskip-0.02cm J \subset C}} 
\frac{\vol_g(C) + \gamma \widehat{T}_k(C)}{\assoc(C)},
\end{align}
where $\widehat{T}_k$ is given as
$\widehat{T}_k(C)  =	\max\left\{ 0,\vol_h(C)-k\right\} $
and $\gamma \hspace{-0.01cm} > \hspace{-0.01cm}\frac{\vol_g(C_{0})\vol(V)} {\theta\,\assoc(C_{0})}$ for a feasible set $C_0\subset V$.
Note that the penalty term is equal to 
 $ \widehat{T}_k(C)  =	\vol_h(C) -  \min\left\{ k,\vol_{h}(C)\right\} ,$
which is a difference of submodular functions.

We could reformulate the seed constraint $J \subset C$ as inequality constraint $|J \cap C| - |J| \ge 0$ and add a similar penalty function to the numerator of \eqref{maxdens_seed_constraint}. However, 
using the structure of the problem, a more direct way to incorporate the seed constraint is possible. It holds that \eqref{maxdens_seed_constraint} has the equivalent form
\begin{align} \label{maxdens_unconstrained}
 \min_{A \subset V \backslash J}  \frac{\vol_g(A) + \vol_g(J) + \gamma \widehat{T}_{k'}(A)}{\assoc(A)+\assoc(J)
 +2\cut(J,A)},
\end{align}
where $k'=k-\vol_h(J)$. 
Solutions $C^*$ of \eqref{maxdens_seed_constraint} and $A^*$ of \eqref{maxdens_unconstrained} are related via $C^* = A^* \cup J$.
In order to derive 
the tight relaxation via Theorem \ref{thm:ratio_setfunct},
we need the Lovasz extension of the 
set functions 
in \eqref{maxdens_unconstrained}. 
 For technical reasons, we replace the constant set functions $\vol_g(J)$ and $\assoc(J)$ 
 by $\vol_g(J)\widehat{P}(A)$ and $\assoc(J)\widehat{P}(A)$,
  respectively, 
  where $\widehat{P}$ is defined as $\widehat{P}(A) = 1$ for $A\neq \emptyset$ and $\widehat{P}(\emptyset) = 0$. 
 This leads to
 the problem
\begin{align} \label{maxdens_unconstrained_mod}
 \min_{A \subset V \backslash J}  \frac{\vol_g(A) + \vol_g(J) \widehat{P}(A) + \gamma \widehat{T}_{k'}(A)}{\assoc(A)+\assoc(J) \widehat{P}(A)
 +2\cut(J,A)}.
\end{align}
The only difference to \eqref{maxdens_unconstrained} lies in the treatment of the empty set. Note that with $\frac{0}{0}:=\infty$ the empty set can never be optimal for problem \eqref{maxdens_unconstrained_mod}. Given an optimal solution $A^*$ of \eqref{maxdens_unconstrained_mod}, one then either considers either $A^* \cup J$ or $J$, depending on whichever has lower objective, which then implies equivalence to \eqref{maxdens_unconstrained}.

The resulting tight relaxation will be a minimization problem over $\R^m$ with $m = |V\backslash J|$ and we assume wlog that the first $m$ vertices of $V$ are the ones in $V\backslash J$. Moreover, we use the notation $f_{\rm max} = \max_{i=1,\dots,m}f_i$ for $f\in \R^m$, and $d^{(A)}_i = \sum_{j\in A} w_{ij}$. The following Lovasz extensions are useful:
\newline
\vspace{-4mm}
\begin{table}[h!]
\label{tab:lovasz_extensions}
\begin{center}
\begin{small}
\begin{sc}
\begin{tabular}{c|c}
Set function & Lovasz extension \\
\hline
$\cut(A,\overline{A})$ & $\frac{1}{2}\sum_{i,j}^m w_{ij} |f_i -f_j|$\\
\hline
$\vol_g(A)$ & $\inner{f,(g_i)_{i=1}^m}$\\
\hline
$\assoc(A)$ & $\inner{f,(d^{(V\backslash J)}_i)_{i=1}^m} - \frac{1}{2}\sum_{i,j}^m w_{ij} |f_i -f_j|$\\
\hline
$\widehat{P}(A)$ & $f_{\rm max}$ \\
\hline
$\widehat{T}_{k'}(A)$ & $\inner{f,(h_i)_{i=1}^m} - T^{(2)}_{k'}(f)$ 
\end{tabular}
\end{sc}
\end{small}
\end{center}
\end{table}
\vspace{-4mm}
\newline
For the sake of brevity, we do not specify the convex function $T^{(2)}_{k'}$. 
Recall from  Section \ref{sec:alg}
that we need only an element of the subdifferential for $T^{(2)}_{k'}$ which by Prop.\;2.2 in \citet{Bach11} is given by 
\begin{align*}
\big(t^{(2)}_{k'}(f)\big)_{j_i} \hskip-0.1cm &=  \hskip-0.1cm \left\{\begin{array}{ll} \hskip-0.1cm 0 & \vol_h(A_{i+1}) >k'\\
\hskip-0.1cm k' - \vol_h(A_{i+1}) & \vol_h(A_{i}) \geq k' ,\\
& \hskip 0.5cm \vol_h(A_{i+1}) \leq k'\\
	\hskip-0.1cm h_{j_i} & \vol_h(A_i)<k'	
		\end{array} \right. ,
\end{align*}
where $j_i$ denotes the index of the $i$-th smallest component of the vector $f$. 
The above Lovasz extensions lead to the following tight relaxation of \eqref{maxdens_unconstrained_mod}:
\begin{align} \label{maxdens_tight_relaxation} 
   	 	\min_{f \in \R^{m}_{+}} 
 			\frac{ R_1(f) - R_2(f)}
 			 {S_1(f)  - S_2(f)} ,
\end{align}
where 
$R_1(f) = \inner{(g_i)_{i=1}^m +\gamma(h_i)_{i=1}^m,f} + \vol_g(J) f_{\rm max}$,
$S_1(f) = \langle (d_i)_{i=1}^m + (d_i^{(J)})_{i=1}^m,f \rangle + \assoc(J) \,f_{\rm max}$,
$R_2(f) = \gamma T^{(2)}_{k'}(f)$ and $S_2(f) = \frac{1}{2}\sum_{i,j}^m w_{ij} |f_i -f_j|.$

\paragraph{Lower bound constraints.} Constraints of the form $\vol_{h}(C) \ge k$ are rewritten as $-\vol_{h}(C) \le -k$, which leads to the penalty term, see \eqref{eq:penality_functions},
 \begin{align*}
 \widehat{T}_k(C) & =\left\{ \begin{array}{cc}
 									\max\left\{ 0,k-\vol_{h}(C)\right\}, & C \neq \emptyset,\\
 									0, & C=\emptyset.\\
 									\end{array} \right. 	 
\end{align*}
The decomposition
$\widehat{T}_k(C)  =	k\ \widehat{P}(C) -  \min\left\{ k,\vol_{h}(C) \right\}$ 
then again yields a difference of submodular functions (noting $k\geq 0$). The derivation then proceeds analogously to the case of upper bound constraints.

\paragraph{Solution via RatioDCA.}

Observe that both numerator and denominator of the tight relaxation \eqref{maxdens_tight_relaxation} are one-homogeneous d.c.\,functions and thus we can apply the RatioDCA of Section \ref{sec:alg}. The crucial step in the algorithm is solving the inner problem (line 3). For both \eqref{maxdens_tight_relaxation} and the tight relaxation of the constrained balanced graph cut problem, it has the form
\begin{align} \label{maxdens_inner_problem}
\min_{\substack{f \in \R^m_+\\ \|f\|_2\le 1}} \{ c_1 f_{\max}  + \inner{f,c_2}  + \lambda^l \frac{1}{2}\sum_{i,j}^m w_{ij} |f_i -f_j| \},
\end{align}
for $c_1 \in \R$ and $c_2 \in \R^m$. We solve this problem 
via the following equivalent dual problem.

\begin{lemma}\label{lem:inner_problem}
The inner problem \eqref{maxdens_inner_problem} is equivalent to
\begin{align*}
- \min_{\substack{\norm{\alpha}_\infty\leq 1\\ \alpha_{ij}=-\alpha_{ji}}} 
	 \min_{v \in S_m} 
		\frac{1}{2}\norm{P_{\R_+^m} \left( - c_1 v -c_2 -\frac{\lambda^l}{2} A \alpha \right)}_2^2
\end{align*}
where $(A\alpha)_i := \sum_{j} w_{ij} (\alpha_{ij}-\alpha_{ji})$, $P_{\R_+^m}$ denotes the projection on the positive orthant
and $S_m$ is the simplex $S_m=\{v \in \R^m \,|\, v_i\geq 0, \sum_{i=1}^m v_i=1\}$.
\end{lemma}
\begin{proof}
First we replace the inner problem \eqref{maxdens_inner_problem} by the modified problem
\begin{align}\label{eq:IP2} 
\min_{f\in \R^m_+} \frac{\lambda^l}{2}\hspace{-0.15cm}\sum_{i,j=1}^m w_{ij}|f_i-f_j| +
c_1 \max_i f_i  + \inner{f,c_2} +\frac{1}{2}\norm{f}_2^2 .
\end{align}
Given a solution $f^ *$ of \eqref{eq:IP2}, a solution of \eqref{maxdens_inner_problem} can be obtained via $f^*/\norm{f^ *}_2$, which can be shown using the $1$-homogeneity of the objective \eqref{maxdens_inner_problem}.
We then derive the dual problem as follows:
\begin{align*}
		 & \min_{f\in \R_+^m}\hspace*{-0.1cm}\frac{\lambda^l}{2}\hspace*{-0.1cm}\sum_{i,j=1}^{m} w_{ij} \abs{f_{i} - f_{j}} + c_1 \max f_i + \inner{f,c_2}  + \frac{1}{2} \norm{f}_2^2 \\
		 & = \min_{f\in \R_+^m} \Big\{\max_{\substack{\norm{\alpha}_\infty\leq 1\\ \alpha_{ij}=-\alpha_{ji}}}\frac{\lambda^l}{2}\sum_{i,j=1}^{m} w_{ij} \left(f_{i} - f_{j}\right) \alpha_{ij}    \\
		 & \hspace{1.5cm}
		   +   \max_{v \in S_m} c_1 \inner{f,v} + \inner{f,c_2} + \frac{1}{2} \norm{f}_2^2 \Big\}  \\
		 & = \max_{\substack{\norm{\alpha}_\infty\leq 1 \\ \alpha_{ij}=-\alpha_{ji} \\ v \in S_m}} 
		 	\min_{f\in \R_+^m} \frac{1}{2} \norm{f}_2^2 + \inner{f, c_1 v + c_2 + \frac{ \lambda^l}{2} A\alpha },		  
\end{align*}
where $(A\alpha)_i := \sum_{j} w_{ij} (\alpha_{ij}-\alpha_{ji})$.
The optimization over $f$ has the solution
\[
	f= P_{\R_+^m} \left( - c_1 v -c_2 - \frac{\lambda^l}{2} A \alpha \right).
\]
Plugging $f$ into the objective and using that $\inner{P_{\R_+^m}(x),x} = \norm{P_{\R_+^m}(x)}_2^2$, we obtain the result.
\end{proof}
This dual problem can be solved efficiently using FISTA \cite{BT09}, a proximal gradient method with guaranteed convergence rate $O(\frac{1}{k^2})$ where $k$ is the number of steps. 
The resulting explicit steps in FISTA with $B_\infty(1)=\{x \in \R\,|\, |x| \le 1\}$ to solve the inner problem are given below.
\floatname{algorithm}{}
\begin{algorithm}[htb] \label{alg:fista}
   \renewcommand{\thealgorithm}{}
   \caption{\textbf{FISTA} for the inner problem}
\begin{algorithmic}[!htb]
   \STATE {\bfseries Input:} Lipschitz constant $L$ of $\nabla \Psi$,
   \STATE {\bfseries Initialization:} $t_1=1$, $\alpha^1 \in \R^{|E|}$,
   \REPEAT
   \STATE            
   		 $v  = \argmin_{u \in S_m} \norm{P_{\R_+^m} \left(  - c_1 u -c_2 -\frac{\lambda^l}{2} A \alpha\right)}_2^2$\\
   		 $z = P_{\R_+^m} \left(- c_1 v - c_2 - \frac{\lambda^l}{2} A \alpha \right)$\\
       $\beta^{k+1}_{rs} 
                           =P_{B_\infty(1)}\left(\alpha^k_{rs} + \frac{1}{L}\lambda^l w_{rs}\big( z_r-z_s\big)\right)$
   \STATE $t_{k+1} = \frac{1+ \sqrt{1+4 t_k^2}}{2}$,
   \STATE $\alpha^{k+1}_{rs} = \beta^{k+1}_{rs} + \frac{t_k-1}{t_{k+1}}\Big(\beta^{k+1}_{rs} - \beta^k_{rs}\Big)$.
   \UNTIL {duality gap $<\epsilon$}
\end{algorithmic}
\end{algorithm}

The most expensive part
of each iteration of the algorithm is a sparse matrix multiplication, which scales linearly in the number of edges.
To solve the first subproblem in FISTA, we make use of the following fact:
\begin{lemma}\label{lem:projection}
Let $x \in \R^n$ and $y:=P_{\R_+^n}(x)$, then 
$\argmin_{v \in {S_n}} 
		\norm{y -  v}_2^2 
  \in \argmin_{v \in S_n} 
				\norm{P_{\R_+^n} \left( x -  v\right)}_2^2$.
\end{lemma}
\begin{proof}
The proof is a straightforward but technical transformation of the KKT optimality conditions of the left problem
into the ones of the right problem.
\end{proof}
Lemma \ref{lem:projection} implies that the minimization problem can be solved via a standard projection onto the simplex, which can be computed in linear time \cite{Kiwiel2007}.

\paragraph{Unconstrained version.} 
In the unconstrained case of the maximum density problem, the tight relaxation \eqref{maxdens_tight_relaxation} reduces to a convex-concave ratio. As remarked in Section \ref{sec:alg} it can then be solved globally optimally with our method, which in this case is equivalent to Dinkelbach's method \cite{Din1967}.   In every iteration, we have to solve
\begin{equation} \label{eq:dinkelbach}
\min_{\substack{f \in \R^n_+\\ \|f\|_\infty\le 1}} \{ \inner{g,f} - \lambda \inner{d,f} 
+ \frac{\lambda}{2} \sum_{i,j=1}^n w_{ij} |f_i -f_j| \}.
\end{equation}
Note that here we used the fact that one can replace the $L_2$ norm constraint in the inner problem by a $L_\infty$ norm constraint, see 
the remark after Prop.\;\ref{prop:decrease_algorithm}.
The following lemma shows that \eqref{eq:dinkelbach} can be rewritten as a $s$-$t$-min-cut-problem, which shows that the procedure is similar to the method of \citet{Gol84}.
\begin{lemma}
Problem \eqref{eq:dinkelbach} is equivalent to the problem
\[ \min_{f_V \in H, \, f_s=1, \, f_t=0} \frac{1}{2}\sum_{i,j \in V'} w'_{ij} |f_i-f_j|,\]
with $V'=V \cup \{s,t\}$, $H:= \left\{ u \in \R^n_+, \norm{u}_\infty \leq 1\right\}$ and some non-negative weights $w'_{ij}$, $i,j\in V'$.
\end{lemma}
\begin{proof}
Note that adding constant terms to the objective does not change the minimizer. We rewrite
\begin{align*}
 &\sum_{i=1}^n g_i(f_i \hskip-0.05cm - \hskip-0.05cm 0)\hskip-0.08cm + \hskip-0.08cm \lambda \hskip-0.08cm \sum_{i=1}^n d_i \hskip-0.08cm - \hskip-0.08cm \lambda  \hskip-0.08cm \sum_{i=1}^n d_i f_i \hskip-0.08cm + \hskip-0.08cm\frac{\lambda}{2} \hskip-0.08cm \sum_{i,j=1}^n w_{ij} |f_i -f_j|\\
=&\sum_{i=1}^n g_i|f_i-0| +\lambda \sum_{i=1}^n d_i |1-f_i| + \frac{\lambda}{2} \sum_{i,j=1}^n w_{ij} |f_i -f_j|,
\end{align*}
where we have used that $f \in H$, where $H:= \left\{ u \in \R^n_+, \norm{u}_\infty \leq 1\right\}$. 
We define the graph as $V'=V \cup \{s,t\}$ and the weight matrix $W'$ with
\[ w'_{ij} = \left\{\begin{tabular}{ll} $\lambda \,w_{ij}$ & if $i,j \in V$,\\ $2\lambda d_j$ & if $i=s$ and $j \in V$,\\ $2 g_i$ & if $i \in V$ and $j=t$, \end{tabular}\right.\]
and can rewrite the problem as
\[ \min_{f_V \in H, \, f_s=1, \, f_t=0} \frac{1}{2}\sum_{i,j \in V'} w'_{ij} |f_i-f_j|,\]
which is a $s$-$t$-mincut.
\end{proof}
The above problem can be efficiently solved, e.g., using the pseudo-flow algorithm of \citet{Hoc2008}.


\begin{table*}[tb!]
\caption{Results for the constrained local normalized cut.
Our solutions (CFSP) always satisfy all constraints and have smaller cuts than the two competing methods LS and LRW.}
\label{tab:results_ncut_short}
\vspace{-1mm}
\begin{center}
\begin{tiny}
\begin{sc}
\begin{tabular}{@{\extracolsep{-2pt}}cccccccc}
                 & Method    & $\le 20\%$      & $\le 40\%$      & $\le 60\%$      & $\le 80\%$      & $\le 100 \%$    & Runtime\\
\hline
CA-GrQc        & LRW       & 0.1311 (0.0686) & 0.1005 (0.0542) & 0.0984 (0.0543) & 0.0920 (0.0439) & 0.0773 (0.0341) & 2\\
(4158,13422)    & LRW+CFSP  & 0.1048 (0.0486) & 0.0695 (0.0318) & 0.0614 (0.0268) & 0.0614 (0.0268) & 0.0457 (0.0217) & 2 + 3\\
                & LS        & 0.2014 (0.0958) & 0.1182 (0.0958) & 0.0685 (0.1089) & 0.0314 (0.0423) & 0.0217 (0.0259) & 6\\
                & LS+CFSP   & 0.1366 (0.0914) & 0.0709 (0.0592) & 0.0340 (0.0494) & 0.0200 (0.0270) & 0.0147 (0.0120) & 6 + 3\\
                & CFSP      & \textbf{0.0315 (0.0292)} & \textbf{0.0157 (0.0131)} & \textbf{0.0138 (0.0115)} & \textbf{0.0083 (0.0055)} & \textbf{0.0069 (0.0044)} & 31\\
\hline
CA-HepTh       & LRW       & 0.2607 (0.0914) & 0.2157 (0.0533) & 0.2015 (0.0498) & 0.1954 (0.0491) & 0.1888 (0.0483) & 9\\
(8638,24806)    & LRW+CFSP  & 0.2074 (0.1003) & 0.1076 (0.0561) & 0.0976 (0.0452) & 0.0882 (0.0305) & 0.0869 (0.0324) & 9 + 8\\
                & LS        & 0.4125 (0.1079) & 0.3439 (0.0631) & 0.3089 (0.0839) & 0.2926 (0.0913) & 0.2778 (0.0923) & 13\\
                & LS+CFSP   & 0.3258 (0.1236) & 0.1894 (0.1126) & 0.1274 (0.0986) & 0.0651 (0.0315) & 0.0618 (0.0324) & 13 + 9\\
                & CFSP      & \textbf{0.0518 (0.0226)} & \textbf{0.0327 (0.0104)} & \textbf{0.0318 (0.0094)} & \textbf{0.0263 (0.0082)} & \textbf{0.0104 (0.0038)} & 58\\
\hline
Cit-HepTh      & LRW       & 0.5052 (0.2208) & 0.4697 (0.2010) & 0.4373 (0.1962) & 0.4067 (0.1998) & 0.3807 (0.2224) & 15\\
(27400,352021)  & LRW+CFSP  & \textbf{0.3888 (0.2261)} & \textbf{0.3249 (0.2072)} & 0.2960 (0.1778) & 0.2528 (0.1689) & 0.2476 (0.1928) & 15 + 368\\
                & LS        & 0.5430 (0.2617) & 0.5099 (0.2524) & 0.4737 (0.2586) & 0.4290 (0.2773) & 0.3997 (0.2834) & 175\\
                & LS+CFSP   & 0.4496 (0.2848) & 0.3585 (0.2185) & 0.3122 (0.2138) & 0.2074 (0.0814) & 0.1772 (0.0782) & 175 + 190\\
                & CFSP      & 0.4693 (0.2676) & 0.3732 (0.2166) & \textbf{0.2683 (0.1494)} & \textbf{0.1748 (0.0683)} & \textbf{0.0752 (0.0233)} & 3704\\
\hline
Cit-HepPh      & LRW       & 0.1784 (0.0541) & 0.1466 (0.0503) & 0.1234 (0.0256) & 0.1079 (0.0120) & 0.1048 (0.0062) & 19\\
(34401,420784)  & LRW+CFSP  & 0.1365 (0.0305) & 0.1132 (0.0201) & 0.1070 (0.0181) & 0.0966 (0.0135) & 0.0948 (0.0052) & 19 + 219\\
                & LS        & 0.1720 (0.0055) & 0.1292 (0.0224) & 0.1155 (0.0147) & 0.1107 (0.0062) & 0.1078 (0.0007) & 103\\
                & LS+CFSP   & 0.1335 (0.0064) & \textbf{0.1064 (0.0114)} & \textbf{0.0965 (0.0091)} & 0.0944 (0.0061) & 0.0916 (0.0011) & 103 + 102\\
                & CFSP      & \textbf{0.1181 (0.0143)} & 0.1127 (0.0101) & 0.1109 (0.0089) & \textbf{0.0928 (0.0039)} & \textbf{0.0913 (0.0015)} & 2666\\
\hline
amazon0302     & LRW       & 0.1768 (0.0833) & 0.1465 (0.0749) & 0.1336 (0.0601) & 0.1221 (0.0504) & 0.1120 (0.0429) & 336\\
(262111,899792) & LRW+CFSP  & 0.1072 (0.0666) & 0.0724 (0.0455) & 0.0577 (0.0419) & 0.0423 (0.0373) & 0.0344 (0.0294) & 336 + 608\\
                & LS        & 0.2662 (0.1204) & 0.2496 (0.1155) & 0.2247 (0.1021) & 0.2066 (0.0892) & 0.1946 (0.0840) & 5765\\
                & LS+CFSP   & 0.1775 (0.0807) & 0.1248 (0.0643) & 0.0923 (0.0675) & 0.0878 (0.0694) & 0.0641 (0.0435) & 5765 + 458\\
                & CFSP      & \textbf{0.0194 (0.0063)} & \textbf{0.0095 (0.0043)} & \textbf{0.0072 (0.0031)} & \textbf{0.0056 (0.0024)} & \textbf{0.0050 (0.0022)} & 3007\\
\hline
amazon0505     & LRW       & 0.2472 (0.1112) & 0.2369 (0.1124) & 0.2249 (0.1132) & 0.2200 (0.1152) & 0.2163 (0.1183) & 210\\
(410236,2439437)& LRW+CFSP  & 0.1058 (0.0833) & 0.0636 (0.0319) & 0.0636 (0.0319) & 0.0636 (0.0319) & 0.0610 (0.0337) & 210 + 2061\\
                & LS        & 0.4124 (0.1751) & 0.3704 (0.1864) & 0.3653 (0.1878) & 0.3576 (0.1919) & 0.3529 (0.1956) & 20558\\
                & LS+CFSP   & 0.1300 (0.0935) & 0.0903 (0.0545) & 0.0782 (0.0587) & 0.0782 (0.0587) & 0.0782 (0.0587) & 20558 + 2900\\
                & CFSP      & \textbf{0.0227 (0.0076)} & \textbf{0.0116 (0.0089)} & \textbf{0.0058 (0.0020)} & \textbf{0.0048 (0.0011)} & \textbf{0.0047 (0.0008)} & 13171\\
\end{tabular}

\end{sc}
\end{tiny}
\end{center}
\vskip -0.1in
\end{table*}
\begin{figure*}[hbtp]
\vskip -0.18in
\caption{Different machine learning communities detected by our algorithm for the highlighted seeds. \textit{Left:} Learning Theory  \textit{Middle:} Sparsity \textit{Right:} Kernels }
\label{fig:results_dblp}
\includegraphics[width=0.315\linewidth]{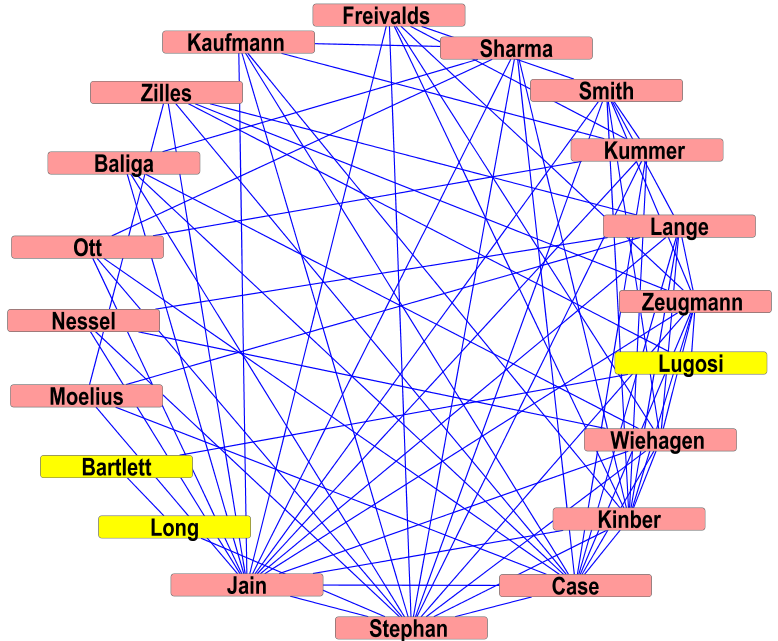}
\hfill
\includegraphics[width=0.315\linewidth]{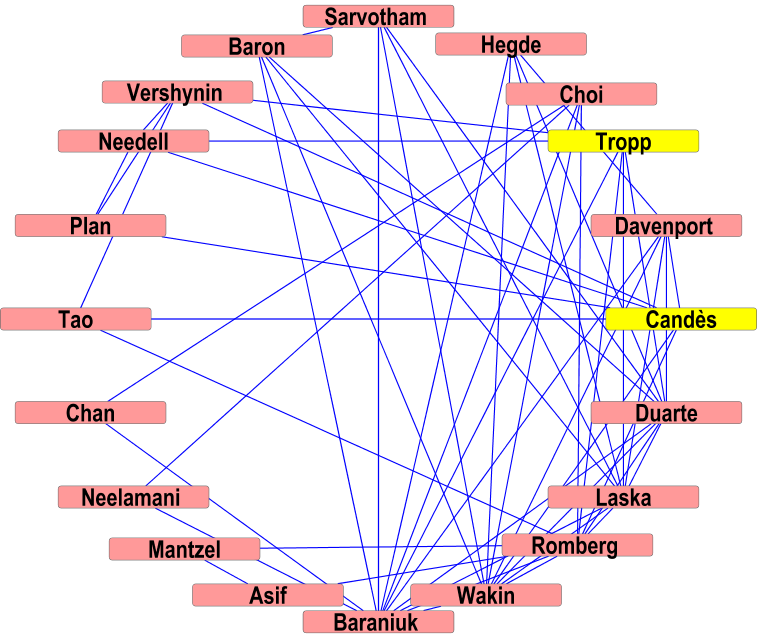}
\hfill
\includegraphics[width=0.315\linewidth]{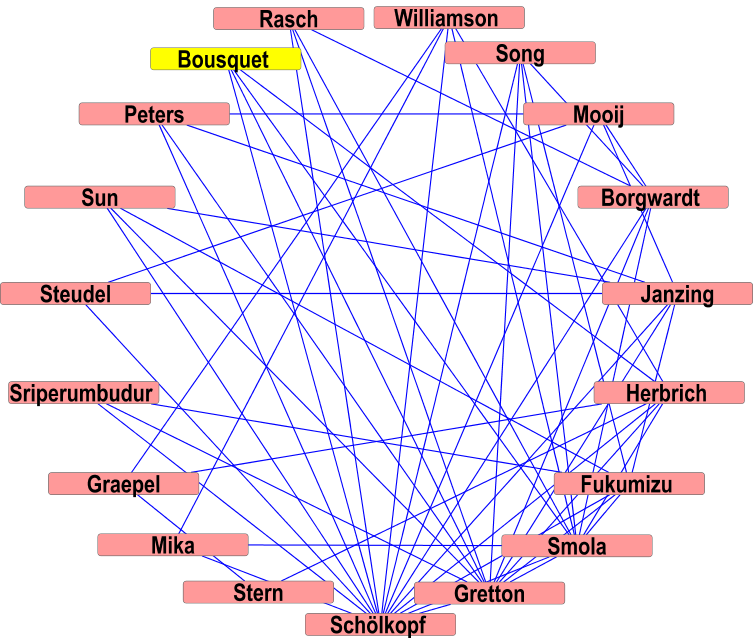}
\vskip -2mm
\end{figure*}

\section{Experiments}\label{sec:experiments}
We 
empirically evaluate the performance of our approach on local clustering and 
  community detection problems.
  Our goal 
  is to address the following questions:
  (i) In terms of the original 
   	objective of the fractional set program, how does the locally optimal solution of our tight relaxation compare to the globally optimal solution of a loose relaxation?
  (ii) How good is our quality guarantee (Theorem \ref{th:quality-guarantee}), i.e.\;how often does our method improve a given sub-optimal solution obtained by another method?
   
  In all experiments we start the RatioDCA with 10
  different random initializations and report the result 
  with smallest objective value. Regarding the parameter $\gamma$ from Theorem \ref{th:setratio_cnstr}, it turns out that best results are obtained by first solving the unconstrained case $(\gamma=0)$ and then increasing $\gamma$ sequentially, until all constraints are fulfilled. In principle, this strategy could also be used to deal with soft or noisy constraints,
  however we focus here on the case of hard constraints.
  
 \paragraph{Local clustering.} 
 	We first consider the local normalized cut problem,
	\begin{align}\label{eq:normalized_cut}
		\min_{\substack{C \subset V\\ s \in C,\ \vol_d(C) \le k} } \frac{\cut(C,\overline{C}) \vol{(V)}}{\vol_d(C)\vol_d(\overline{C})},\ 
	 \end{align}
where $s \in V$ is a given seed vertex. We evaluate our approach (denoted as CFSP) against the Local Spectral (LS) method
         by \citet{MahOreVis12} and the Lazy Random Walk (LRW) by \citet{AndLan06} 
	 on large social networks of the Stanford Large Network Dataset Collection \cite{LeskovecSnap}. 
	 
\begin{table*}[hbt!]
\caption{Constrained local normalized Cheeger cuts of the solutions 
obtained by our method (note that we optimized the normalized cut) 
as well as the solutions of Lazy Random Walk (LRW) where we threshold in each step according to the normalized Cheeger cut objective}
\label{tab:results_ncc_long}
\vspace{-1mm}
\begin{center}
\begin{tiny}
\begin{sc}
\begin{tabular}{cccccccc}
                 & Method    & $\le 20\%$      & $\le 40\%$      & $\le 60\%$      & $\le 80\%$      & $\le 100 \%$    & Runtime (sec)\\
\hline
CA-GrQc         & LRW       & 0.1298 (0.0677) & 0.0992 (0.0536) & 0.0967 (0.0537) & 0.0894 (0.0418) & 0.0753 (0.0340) & 1\\
                & CFSP      & \textbf{0.0312 (0.0289)} & \textbf{0.0153 (0.0128)} & \textbf{0.0133 (0.0110)} & \textbf{0.0079 (0.0051)} & \textbf{0.0064 (0.0040)} & 31\\
\hline
CA-HepTh        & LRW       & 0.2601 (0.0911) & 0.2150 (0.0530) & 0.2005 (0.0495) & 0.1941 (0.0488) & 0.1873 (0.0481) & 1\\
                & CFSP      & \textbf{0.0517 (0.0225)} & \textbf{0.0326 (0.0104)} & \textbf{0.0317 (0.0093)} & \textbf{0.0261 (0.0082)} & \textbf{0.0103 (0.0037)} & 58\\
\hline
Cit-HepTh       & LRW       & 0.4967 (0.2300) & 0.4565 (0.2150) & 0.4179 (0.2174) & 0.3890 (0.2174) & 0.3705 (0.2307) & 10\\
                & CFSP      & \textbf{0.4673 (0.2690)} & \textbf{0.3712 (0.2176)} & \textbf{0.2661 (0.1496)} & \textbf{0.1681 (0.0706)} & \textbf{0.0705 (0.0150)} & 3704\\
\hline
Cit-HepPh       & LRW       & 0.1574 (0.0497) & 0.1104 (0.0364) & \textbf{0.0769 (0.0151)} & 0.0573 (0.0064) & \textbf{0.0566 (0.0062)} & 14\\
                & CFSP      & \textbf{0.1168 (0.0156)} & \textbf{0.1067 (0.0138)} & 0.0986 (0.0202) & \textbf{0.0500 (0.0098)} & 0.0584 (0.0049) & 2666\\
\hline
amazon0302      & LRW       & 0.1768 (0.0833) & 0.1464 (0.0749) & 0.1335 (0.0600) & 0.1220 (0.0503) & 0.1118 (0.0428) & 241\\
                & CFSP      & \textbf{0.0193 (0.0063)} & \textbf{0.0095 (0.0043)} & \textbf{0.0072 (0.0031)} & \textbf{0.0056 (0.0024)} & \textbf{0.0050 (0.0022)} & 3007\\
\hline
amazon0505      & LRW       & 0.2472 (0.1111) & 0.2369 (0.1124) & 0.2248 (0.1132) & 0.2200 (0.1152) & 0.2162 (0.1183) & 289\\
                & CFSP      & \textbf{0.0227 (0.0076)} & \textbf{0.0116 (0.0089)} & \textbf{0.0058 (0.0020)} & \textbf{0.0048 (0.0011)} & \textbf{0.0047 (0.0008)} & 13171\\
\end{tabular}

\end{sc}
\end{tiny}
\end{center}
\vskip -0.2in
\end{table*}
	 
	In \citet{MahOreVis12}, a spectral-type relaxation is derived for \eqref{eq:normalized_cut} that can be 
	solved globally optimally. The resulting continuous solution is then transformed into a set via 
	optimal thresholding. However, contrary to our method this is not guaranteed
	to yield a set that satisfies 
	both the seed and volume constraints. 
	Hence \citet{MahOreVis12} suggest, at the cost of losing 
	their 
	approximation guarantees,
	 to perform \textit{constrained} optimal thresholding which 
	considers only thresholds that yield feasible sets. 
In a recent generalization of their work, \citet{HanMah2012} compute a sequence of locally-biased eigenvectors, the first of which corresponds to the solution of the spectral-type relaxation of \citet{MahOreVis12}. We use the code of 
 \citet{HanMah2012} to compute the solution of LS in our experiments.
	The local clustering technique of \citet{AndLan06} 
	 explores the graph locally by performing a lazy random walk with the transition matrix $M= \frac{1}{2} \left(I + WD^{-1}\right)$, where $D$ is the degree matrix of the graph and the initial distribution is concentrated on the seed set. 
Under some conditions on the seed set, it is shown that after a specified number of steps optimal thresholding of the random walk vector yields a set with ``good'' normalized Cheeger cut.
 However, they cannot guarantee 
 that the resulting set contains the seed. For a fair comparison, we compute the full sequence of random walk vectors until the stationary distribution is reached, and in each step perform constrained optimal thresholding according to the normalized cut 
 objective.  

	For each dataset we generate 10 random seeds.
In order to ensure that meaningful intervals for the volume constraint are explored, 
	we first 
	solve the local clustering problem only with the seed constraint. 
	Treating this as the ``unconstrained'' solution $C_{0}$, 
    we then repeat the experiment with upper bounds of the form  $\vol(C)\leq \alpha \vol(C_0)$, where $\alpha \in \{0.2,0.4,0.6,0.8\}$.
	
Table \ref{tab:results_ncut_short} shows mean and standard deviation of the normalized cut values averaged over the 10 different random trials (seeds) 
and average runtime over the different runs and volume constraints. To demonstrate the quality guarantee (Theorem \ref{th:quality-guarantee}) 
	we also initialize CFSP
	with the solution of LS and LRW. Our method CFSP consistently outperforms the competing methods by large margins 
	and always finds solutions that satisfy all constraints. 
       In some cases CFSP initialized with LS or LRW outperforms CFSP with 10 random initializations. While LRW is very fast,
       the obtained normalized cuts are far from being competitive. 
Note that 
CFSP still performs better  
       if one uses for the optimal thresholding the normalized Cheeger cut for which LRW has been designed.
This is shown in Table \ref{tab:results_ncc_long} where we compare the normalized Cheeger cut of our solutions (note that we optimized the normalized cut) to the solution obtained by the Lazy Random Walk method where we threshold in each step according to the normalized Cheeger cut objective.

\paragraph{Community detection.}
   We evaluate our approach for local community detection according to \eqref{eq:maxdens_general_only_upper}.
       The task is to extract communities around given seed sets in a co-author network constructed from the
       DBLP publication database. Each node in the network represents a researcher and an edge between two nodes indicates
       a common publication. The weights of the graph are defined as
	$w_{ij}=\sum_{l \in P_i \cap P_j} \frac{1}{\abs{A_l}}$, where $P_i,P_j$ denotes the set of publications of authors $i$ and $j$ and $A_l$ denote the sets of authors for publication $l$, i.e.\,the weights represent the total contribution to shared papers.
	This normalization avoids the problem of giving high weight to a researcher who 
	has 
	publications 
	that have a large number of authors,
	which usually does not reflect close collaboration with all co-authors.
	
	To avoid finding a trivial densely connected group of researchers with 
	few connections to the rest of the authors,
	we further restrict the graph by considering only authors 
	with 
	at least two 
	publications and maximum distance two from the seed set. 
        As volume function in \eqref{eq:maxdens_general_only_upper}, we use the volume
        of the original graph in order to further enforce densely connected components.
     
     We perform local community detection  
        with the size constraint $|C| \leq 20$ and three different 
        seed sets $J_1=\left\{\text{P. Bartlett}\right.$, P. Long,  $\left.\text{G. Lugosi}\right\}$, $J_2=\left\{\text{E. Candes}\right.$, $\left.\text{J. Tropp}\right\}$ and 
$J_3=\left\{\text{O. Bousquet}\right\}$. 
$J_1$ consists of well-known researchers in learning theory, and all members of the detected community work in this area.
To validate this, we counted the number of publications in the two main theory conferences COLT and ALT.
On average each author has 18.2 publications in these two conferences
(see Table \ref{tab:theory_pubs} for more details).
\begin{table}
\caption{The number of publications in ALT and COLT of each author in the ``learning theory'' community found}
\label{tab:theory_pubs}
\begin{center}
\begin{tabular}{|c|c|c|}
\hline
	Author		&	COLT &	ALT\\
\hline			
    Sandra Zilles &		3 &	13\\
    Peter L. Bartlett &		24 &	2\\
    Carl H. Smith		&13&	4\\
    Philip M. Long		&21	&3\\
    John Case			&12	&18\\
    Sanjay Jain		&21&	40\\
    Steffen Lange	&	14&	5\\
    Rolf Wiehagen	&	6	&7\\
    Thomas Zeugmann &		6&	20\\
    Rusins Freivalds	&	6&	5\\
    Efim B. Kinber	&	11&	9\\
    Frank Stephan	&	13	&28\\
    Martin Kummer	&	5&	0\\
    Arun Sharma	&	10&	13\\
    Samuel E. Moelius  &		1&	5\\
    Gabor Lugosi	&	16&	1\\
    Matthias Ott	 &	2	&1\\
    Jochen Nessel	&	1&	2\\
    Susanne Kaufmann  &		1	&1\\
    Ganesh Baliga	&	1&	0\\
    \hline
    \end{tabular}
\end{center}
\end{table}
The seeds $J_{2}$ yield a community of key scientists in the field of sparsity such as 
T. Tao, R. Baraniuk, J. Romberg, M. Wakin, R. Vershynin etc. The third community 
contains researchers
 who either are/were members of the group of B. Sch\"olkopf or have closely collaborated with his group. 
 
\section*{Acknowledgements}
This work has been supported by DFG Excellence Cluster
MMCI and ERC Starting Grant NOLEPRO.

\bibliographystyle{icml2013}
\bibliography{literatur}

\end{document}